\pdfoutput=1
\documentclass{article}

\usepackage[utf8]{inputenc} 
\usepackage[T1]{fontenc}    

\usepackage[bitstream-charter]{mathdesign}
\usepackage{amsmath}
\usepackage[scaled=0.92]{PTSans}

\usepackage[
  paper  = letterpaper,
  left   = 1.65in,
  right  = 1.65in,
  top    = 1.0in,
  bottom = 1.0in,
  ]{geometry}

\usepackage[usenames,dvipsnames,table]{xcolor}
\definecolor{shadecolor}{gray}{0.9}

\usepackage[final,expansion=alltext]{microtype}
\usepackage[english]{babel}
\usepackage[parfill]{parskip}
\usepackage{afterpage}
\usepackage{framed}

{\endMakeFramed}

\DeclareRobustCommand{\parhead}[1]{\textbf{#1}~}

\usepackage{lineno}

\usepackage{ragged2e}

\newcounter{parcount}

\usepackage{graphicx}
\usepackage{wrapfig}
\usepackage[labelfont=bf,font=footnotesize,width=.9\textwidth]{caption}
\usepackage[format=hang]{subcaption}

\usepackage{booktabs,multirow,multicol}       

\usepackage[algoruled]{algorithm2e}
\usepackage{listings}
\usepackage{fancyvrb}
\fvset{fontsize=\normalsize}

\usepackage[colorlinks,linktoc=all]{hyperref}
\usepackage[all]{hypcap}
\hypersetup{citecolor=MidnightBlue}
\hypersetup{linkcolor=MidnightBlue}
\hypersetup{urlcolor=MidnightBlue}

\usepackage[nameinlink]{cleveref}

\usepackage[acronym,smallcaps,nowarn]{glossaries}

\lstdefinestyle{mystyle}{
    commentstyle=\color{OliveGreen},
    keywordstyle=\color{BurntOrange},
    numberstyle=\tiny\color{black!60},
    stringstyle=\color{MidnightBlue},
    basicstyle=\ttfamily,
    breakatwhitespace=false,
    breaklines=true,
    captionpos=b,
    keepspaces=true,
    numbers=left,
    numbersep=5pt,
    showspaces=false,
    showstringspaces=false,
    showtabs=false,
    tabsize=2
}
\usepackage{centernot}
\usepackage{amsthm}
\usepackage{amsfonts}       
\usepackage{nicefrac}       
\usepackage{mathtools}
\usepackage{amsbsy}
\usepackage{amstext}
\usepackage{amsthm}
\usepackage{thmtools}
\usepackage{thm-restate}

\begingroup
    \makeatletter
    \@for\theoremstyle:=definition,remark,plain\do{%
        \expandafter\g@addto@macro\csname th@\theoremstyle\endcsname{%
            \addtolength\thm@preskip\parskip
            }%
        }
\endgroup

\crefname{lemma}{lemma}{lemmas}
\Crefname{lemma}{Lemma}{Lemmas}
\crefname{thm}{theorem}{theorems}
\Crefname{thm}{Theorem}{Theorems}
\crefname{prop}{proposition}{propositions}
\Crefname{prop}{Proposition}{Propositions}
\crefname{assumption}{assumption}{assumptions}
\crefname{assumption}{Assumption}{Assumptions}

\newcommand\dif{\mathop{}\!\mathrm{d}}

\newcommand{\cN}{\mathcal{N}}

\newcommand{\g}{\, | \,}

\newcommand{\E}[2]{\mathbb{E}_{#1}\left[#2\right]}

\usepackage{booktabs,arydshln}
\makeatletter
\def\adl@drawiv#1#2#3{%
        \hskip.5\tabcolsep
        \xleaders#3{#2.5\@tempdimb #1{1}#2.5\@tempdimb}%
                #2\z@ plus1fil minus1fil\relax
        \hskip.5\tabcolsep}
\newcommand{\cdashlinelr}[1]{%
  \noalign{\vskip\aboverulesep
           \global\let\@dashdrawstore\adl@draw
           \global\let\adl@draw\adl@drawiv}
  \cdashline{#1}
  \noalign{\global\let\adl@draw\@dashdrawstore
           \vskip\belowrulesep}}
\makeatother

\renewcommand{\epsilon}{\varepsilon}

\declaretheorem[style=plain,numberwithin=section,name=Theorem]{theorem}

\declaretheorem[style=definition,sibling=theorem,name=Example]{example}

\newenvironment{example*}
 {\pushQED{\qed}\example}
 {\popQED\endexample}
\numberwithin{equation}{section}

\newcommand{\Reals}{\mathbb{R}}

\DeclareMathOperator*{\argmin}{argmin}

\DeclareMathOperator*{\logit}{logit}

\newcommand{\EE}{\mathbb{E}}
\renewcommand{\Pr}{\mathrm{P}}

\newcommand{\given}{\mid}

\newcommand{\abs}[1]{\left\lvert#1 \right\rvert}

\newcommand{\dist}{\ \sim\ }
\newcommand{\distiid}{\overset{\mathrm{iid}}{\dist}}

\newcommand{\cdo}{\mathrm{do}} 

\newcommand{\bernDist}{\mathrm{Bern}}

\providecommand\given{} 
\newcommand\SetSymbol[1][]{
  \nonscript\,#1:\nonscript\,\mathopen{}\allowbreak}
\DeclarePairedDelimiterX\Set[1]{\lbrace}{\rbrace}%
{ \renewcommand\given{\SetSymbol[]} #1 }

\usepackage[%
minnames=1,maxnames=99,maxcitenames=2,
style=alphabetic,
doi=false,
url=false,
firstinits=true,
hyperref,
natbib,
backend=bibtex,
sorting=nyt
]{biblatex}%

\newbibmacro*{journal}{%
  \iffieldundef{journaltitle}
    {}
    {\printtext[journaltitle]{%
       \printfield[noformat]{journaltitle}%
       \setunit{\subtitlepunct}%
       \printfield[noformat]{journalsubtitle}}}}

\DeclareFieldFormat[article,inbook,incollection,inproceedings,patent,thesis,unpublished]{titlecase}{\MakeSentenceCase*{#1}}

\AtEveryBibitem{%
\ifentrytype{article}{
    \clearfield{url}%
    \clearfield{urldate}%
    \clearfield{eprint}
}{}
\ifentrytype{book}{
    \clearfield{url}%
    \clearfield{urldate}%
}{}
\ifentrytype{collection}{
    \clearfield{url}%
    \clearfield{urldate}%
}{}
\ifentrytype{incollection}{
    \clearfield{url}%
    \clearfield{urldate}%
}{}
}

\AtEveryBibitem{
    \clearfield{pages}
    \clearfield{review}%
    \clearfield{series}
    \clearfield{volume}
    \clearfield{pages}
    \clearfield{month}
    \clearfield{eprint}
    \clearfield{isbn}
    \clearfield{issn}
    \clearlist{location}
    \clearfield{series}
    \clearlist{publisher}
    \clearname{editor}
}{}

\crefformat{equation}{(#2#1#3)}
\crefformat{figure}{Figure~#2#1#3}
\crefname{example}{Example}{Examples}
\crefname{lemma}{Lemma}{Lemmas}
\crefname{cor}{Corollary}{Corollaries}
\crefname{theorem}{Theorem}{Theorems}
\crefname{assumption}{Assumption}{Assumptions}

\usepackage{enumitem} 
\usepackage[separate-uncertainty=true,multi-part-units=single]{siunitx} 

\newcommand{\maxf}[1]{{\cellcolor[gray]{0.8}} #1}
\global\long\def\embedding{\lambda}
\global\long\def\globparam{\gamma}
\newcommand{\psiaiptw}{\hat{\psi}^{\mathrm{A}}}

\usepackage{xcolor}

\definecolor{WowColor}{rgb}{.75,0,.75}
\definecolor{SubtleColor}{rgb}{0,0,.50}

\newcommand{\LATER}[1]{\textcolor{SubtleColor}{ {\tiny \bf ($\dagger$)} #1}}
\newcommand{\TBD}[1]{\textcolor{SubtleColor}{ {\tiny \bf (!)} #1}}
\newcommand{\PROBLEM}[1]{\textcolor{WowColor}{ {\bf (!!)} {\bf #1}}}

\newcounter{margincounter}
\newcommand{\displaycounter}{{\arabic{margincounter}}}
\newcommand{\incdisplaycounter}{{\stepcounter{margincounter}\arabic{margincounter}}}

\newcommand{\fTBD}[1]{\textcolor{SubtleColor}{$\,^{(\incdisplaycounter)}$}\marginpar{\tiny\textcolor{SubtleColor}{ {\tiny $(\displaycounter)$} #1}}}

\newcommand{\fPROBLEM}[1]{\textcolor{WowColor}{$\,^{((\incdisplaycounter))}$}\marginpar{\tiny\textcolor{WowColor}{ {\bf $\mathbf{((\displaycounter))}$} {\bf #1}}}}

\newcommand{\fLATER}[1]{\textcolor{SubtleColor}{$\,^{(\incdisplaycounter\dagger)}$}\marginpar{\tiny\textcolor{SubtleColor}{ {\tiny $(\displaycounter\dagger)$} #1}}}

\renewcommand{\LATER}[1]{}
\renewcommand{\fLATER}[1]{}
\renewcommand{\TBD}[1]{}
\renewcommand{\fTBD}[1]{}
\renewcommand{\PROBLEM}[1]{}
\renewcommand{\fPROBLEM}[1]{}

\usepackage[affil-it]{authblk}

\title{Using Embeddings to Correct for Unobserved Confounding in Networks}
\date{}
\author[1]{Victor Veitch}
\author[1]{Yixin Wang}
\author[1,2]{David M. Blei}
\affil[1]{Department of Statistics, Columbia University}
\affil[2]{Department of Computer Science, Columbia University}

\begin{document}
\maketitle

\begin{abstract}
  We consider causal inference in the presence of unobserved confounding.
  We study the case where a proxy is
  available for the unobserved confounding in the form of a network connecting
  the units.
  For example, the link structure of a social network carries information
  about its members.
  We show how to effectively use the proxy to do causal inference.
  The main idea is to reduce the causal estimation problem to a
  semi-supervised prediction of both the treatments and outcomes.
  Networks admit high-quality embedding models that can
  be used for this semi-supervised prediction.
  We show that the method yields valid inferences under suitable (weak) conditions on the quality of
  the predictive model.  We validate the method with experiments on a
  semi-synthetic social network dataset.
  Code is available at \href{https://github.com/vveitch/causal-network-embeddings}{github.com/vveitch/causal-network-embeddings}.  
\end{abstract}

\section{Introduction}

We consider causal inference in the presence of unobserved
confounding, i.e., where unobserved variables may affect both the
treatment and the outcome.  We study the case where there is an
observed proxy for the unobserved confounders, but (i) the proxy has
non-iid structure, and (ii) a well-specified generative model for the
data is not available.

\begin{example*}
We want to infer the efficacy of a drug based on observed outcomes of
people who are connected in a social network.  Each unit $i$ is a
person.  The treatment variable $t_{i}$ indicates whether they took
the drug, a response variable $y_{i}$ indicates their health outcome,
and latent confounders $z_{i}$ might affect the treatment or
response. For example, $z_{i}$ might be unobserved age or sex.  We
would like to compute the average treatment effect, controlling for
these confounds.  We assume the social network itself is associated
with $z$, e.g., similar people are more likely to be friends. This
means that the network itself may implicitly contain confounding
information that is not explicitly collected.
\end{example*}

In this example, inference of the causal effect would be
straightforward if the confounder $z$ were available.  So,
intuitively, we would like to infer substitutes for the latent $z_{i}$
from the underlying social network structure. Once inferred, these
estimates $\hat{z}_{i}$ could be used as a substitute for $z_{i}$ and we could estimate the causal effect \cite{Shalizi:McFowland:2016}.

For this strategy to work, however, we need a well-specified
generative model (i.e., joint probability distribution) for $z$ and
the full network structure.  But typically no such model is
available. For example, generative models of networks with latent unit
structure---such as stochastic block models \cite{Wang:Wong:1987,Airoldi:Blei:Fienberg:Xing:2008}
or latent space models \cite{Hoff:Handcock:2002}---miss 
properties of real-world networks
\cite{Durrett:2006,Newman:2009,Orbanz:Roy:2015}.
Causal estimates based on substitutes inferred from misspecified models are inherently suspect.

Embedding methods offer an alternative to fully specified generative
models. Informally, an embedding method assigns a real-valued
embedding vector $\hat{\embedding}_{i}$ to each unit, with the aim
that conditioning on the embedding should decouple the properties of
the unit and the network structure.  For example, $\hat{\lambda}_{i}$ might
be chosen to explain the local network structure of user $i$.

The embeddings are learned by minimizing an objective function over
the network, with no requirement that this objective correspond to any
generative model.  For pure predictive tasks, e.g., classification of
vertices in a graph, embedding-based approaches are state of the art
for many real-world datasets
\citep[e.g.,][]{Perozzi:Al-Rfou:Skiena:2014,Chamberlain:Clough:Deisenroth:2017,Hamilton:Ying:Leskovec:2017:review,Hamilton:Ying:Leskovec:2017:inductive,Veitch:Austern:Zhou:Blei:Orbanz:2018}.
This suggests that network embeddings might be usefully adapted to the
inference of causal effects.


The method we develop here stems from the following insight.  Even if
we knew the confounders $\{z_{i}\}$ we would not actually use all the
information they contain to infer the causal effect. Instead, if we
use estimator $\hat{\psi}_{n}$ to estimate the effect $\psi$, then we
only require the part of $z_i$ that is actually used by
the estimator $\hat{\psi}_{n}$. 
For example, if $\hat{\psi}_{n}$ is an inverse
probability weighted estimator \cite{Cole:Hernan:2008} then we require
only estimates for the propensity scores $\Pr(T_{i}=1\given z_{i})$
for each unit.


What this means is that if we can build a good \emph{predictive} model
for the treatment then we can plug the outputs into a causal effect estimate directly, 
without any need to learn the true $z_{i}$.
The same idea applies generally by using a predictive model for both the treatment and outcome.
Reducing the causal inference problem to a predictive
problem is the crux of this paper.
It allows us to replace the assumption of a well-specified model
with the more palatable assumption that the black-box embedding method
produces a strong predictor.

The contributions of this paper are:
\begin{itemize}[nosep]
\item a procedure for estimating treatment effects using network embeddings;
\item an extension of robust estimation results to (non-iid) network data, showing the method yields valid estimates under weak conditions;
\item and, an empirical study of the method on social network data. 
\end{itemize}

\section{Related Work}
Our results connect to a number of different areas.

\parhead{Causal Inference in Networks.} 
Causal inference in networks has attracted significant attention \citep[e.g.,][]{Shalizi:McFowland:2016,TchetgenTchetgen:Fulcher:Shpitser:2017,Ogburn:Sofrygin:Diaz:vanderLaan:2017,Ogburn:VanderWeele:2017,Ogburn:2018}.
Much of this work is aimed at inferring the causal effects of treatments applied using the network;
e.g., social influence or contagion.
A major challenge in this area is that homophily---the tendency of similar people to cluster in a network---is
generally confounded with contagion---the influence people have on their neighbors \cite{Shalizi:Thomas:2011}.
In this paper, we assume that each person's treatment and outcome are independent of the network
once we know that person's latent attributes; i.e., we assume pure homophily.
This is a reasonable assumption in some situations, but certainly not all.
Our major motivation is simply that pure homophily is the simplest case,
and is thus the natural proving ground for the use of black-box methods in causal network problems.
It is an import future direction to extend the results developed here to the contagion case.

\citet{Shalizi:McFowland:2016} address the homophily/contagion issue with a two-stage estimation procedure.
They first estimate latent confounders (node properties),
then use these in a regression based estimator in the second stage.
Their main result is a proof that if the network was actually generated by either a stochastic block model or a latent space model
then the estimation procedure is valid.
Our main motivation here is to avoid such  well-specified model assumptions.  
Their work is complementary to our approach: we impose a weaker assumption, but we only address homophily.

\parhead{Causal Inference Using Proxy Confounders.} 
Another line of connected research deals with causal inference with hidden confounding when 
there is an observed proxy for the confounder \cite{Kuroki:Miyakawa:1999, Pearl:2012, Kuroki:Pearl:2014, Miao:Geng:TchetgenTchetgen:2018, Louizos:Shalit:Mooij:Sontag:Zemel:Welling:2017}.
This work assumes the data is generated 
independently and identically as $(X_i, Z_i, T_i, Y_i) \distiid P$ for some data generating distribution $P$.
The variable $Z_i$ causally affects $T_i$, $Y_i$, and $X_i$.
The variable(s) $X_i$ are interpreted as noisy versions of $Z_i$.
The main question here is when the causal effect is (non-parametrically) identifiable. 
The typical flavor of the results is: if the proxy distribution satisfies certain conditions then the marginal distribution $P(Z_i,T_i,Y_i)$ is identifiable,
and thus so too is the causal effect.
The main difference with the problem we address here is that we consider proxies with non-iid structure and
we do not demand recovery the true data generating distribution.



\parhead{Double machine learning.} 
\citet{Chernozhukov:Chetverikov:Demirer:Duflo:Hansen:Newey:Robins:2017} addresses robust estimation of
causal effects in the i.i.d. setting. 
Mathematically, our main estimation result, \cref{thm:CAN}, is a fairly straightforward adaptation of their result.
The important distinction is conceptual: we treat a different data generating scenario. 

\parhead{Embedding methods.} \citet{Veitch:Sridhar:Blei:2019} use the strategy of reducing causal estimation to prediction
to harness text embedding methods for causal inference with text data.

\section{Setup}

We first fix some notation and recall some necessary ideas about the
statistical estimation of causal effects. We take each statistical
unit to be a tuple $O_i=(Y_{i},T_{i},Z_{i})$, where $Y_{i}$ is the
response, $T_{i}$ is the treatment, and $Z_{i}$ are
(possibly confounding) unobserved attributes of the units.
We assume that the units are drawn independently and
identically at random from some distribution $P$, i.e.,
$O_{i}\distiid P$.
We study the case where there is a network connecting the units.
We assume that the treatments and outcomes are independent of the network
given the latent attributes $\{Z_{i}\}$.
This condition is implied by the (ubiquitous) exchangeable network assumption
\cite{Orbanz:Roy:2015,Veitch:Roy:2015,Crane:Dempsey:2016}, though
our requirement is weaker than exchangeability.


The average treatment effect of a binary outcome is defined as
\[
\psi=\EE[Y\given\cdo(T=1)]-\EE[Y\given\cdo(T=0)].
\]
The use of Pearl's $\cdo$ notation indicates that the effect of interest
is causal: what is the expected outcome if we intervene by assigning the treatment
to a given unit? If $Z_{i}$ contains all common influencers (a.k.a. confounders) of $Y_{i}$
and $T_{i}$ then the causal effect is identfiable as a parameter
of the observational distribution:
\begin{equation}
\psi=\EE[\EE[Y\given Z,T=1]-\EE[Y\given Z,T=0]].\label{eq:psi_def}
\end{equation}

Before turning to the unobserved $Z$ case, 
we recall some ideas from the case where $Z$ is observed.
Let $Q(t,z)=\EE[Y\given t,z]$
be the conditional expected outcome, and $\hat{Q}_{n}$ be an estimator for this function. 
Following \ref{eq:psi_def}, a natural choice of estimator $\hat{\psi}_{n}$ is:
\begin{equation*}
\hat{\psi}_{n}^{Q}=\frac{1}{n}\sum_{i}\left[\hat{Q}_{n}(1,z_{i})-\hat{Q}_{n}(0,z_{i})\right].
\end{equation*}
That is, $\psi$ is estimated by a two-stage procedure: First, produce
an estimate for $\hat{Q}_{n}$.
Second, plug $\hat{Q}_{n}$ into a pre-determined statistic to compute the estimate.

Of course, $\hat{\psi}_{n}^{Q}$ is not the only possible
choice of estimator. In principle, it is possible to do better by
incorporating estimates $\hat{g}_n$ of the propensity scores $g(z)=\Pr(T=1\given z)$.
The augmented inverse probability of treatment weighted (A-IPTW) estimator $\psiaiptw_{n}$
is an important example \cite{Robins:Rotnitzky:vanderLaan:2000,Robins:2000}:
\begin{equation}
\label{eq:psiaiptw-canon}
\psiaiptw_{n} = \frac{1}{n}\sum_{i}\hat{Q}_{n}(1,z_{i})-\hat{Q}_{n}(0,z_{i})
+\frac{1}{n}\sum_{i}\left(\frac{I[t_{i}=1]}{\hat{g}_{n}(z_{i})}-\frac{I[t_{i}=0]}{1-\hat{g}_{n}(z_{i})}\right)(y_{i}-\hat{Q}_{n}(t_{i},z_{i})).
\end{equation}
We call $\eta(z) = (Q(0,z), Q(1,z), g(z))$ the nuisance parameters.
The main advantage of $\psiaiptw_{n}$ is that
it is robust to misestimation of the nuisance parameters \cite{Robins:Rotnitzky:Zhao:1994,vanderLaan:Rose:2011,Chernozhukov:Chetverikov:Demirer:Duflo:Hansen:Newey:Robins:2017}.
For example, it has the double robustness property: $\hat{\psi}_{n}$
is consistent if either $\hat{g}_{n}$ or $\hat{Q}_{n}$ is consistent.
If both are consistent, then $\psiaiptw_{n}$ is
the asymptotically most efficient possible estimator~\cite{Bickel:Klaassen:Ritov:Wellner:2000}. 
We will show below that the good theoretical properties of the suitably modified A-IPTW estimator
persist for the embedding method even in the non-iid setting of this paper.

There is a remaining complication. In the general case, if the same
data $\boldsymbol{O}_{n}$ is used to estimate $\hat{\eta}_{n}$ and
to compute $\hat{\psi}_{n}(\boldsymbol{O}_{n};\hat{\eta}_{n})$ then
the estimator is not guaranteed to maintain good asymptotic properties.
This problem can be solved by splitting the data, using one part to estimate $\hat{\eta}_{n}$
and the other to compute the estimate \cite{Chernozhukov:Chetverikov:Demirer:Duflo:Hansen:Newey:Robins:2017}.
We rely on this data splitting approach.

\section{Estimation}
\label{sec:estimation}

\global\long\def\samp{\mathsf{Sample}}
\global\long\def\crossent{\mathsf{CrossEntropy}}

We now return to the setting where the $\{z_{i}\}$ are unobserved,
but a network proxy is available.

Following the previous section, we want to hold out
a subset of the units $i\in I_{0}$ and, for each of these units, produce estimates
of the propensity score $g(z_{i})$ and the conditional expected outcome $Q(t_{i},z_{i})$.
Our starting point is (an immediate corollary of) \citep[][Thm.~3]{Rosenbaum:Rubin:1983}:
\begin{theorem}\label{thm:prediction-suffices}
  Suppose $\lambda(z)$ is some function of the latent attributes such that
  at least one of the following is $\lambda(Z)$-measurable:
  (i) $(Q(1,Z), Q(1,Z))$, or (ii) $g(Z)$.
  If adjusting for $Z$ suffices to render the average treatment effect identifiable
  then adjusting for only $\lambda(Z)$ also suffices.
  That is, $\psi=\EE[\EE[Y\given \lambda(Z), T=1]-\EE[Y\given \lambda(Z), T=0]]$
\end{theorem}
The significance of this result is that
adjusting for the confounding effect of the latent attributes
does not actually require us to recover the latent attributes.
Instead, it suffices to recover only the aspects $\lambda(z_i)$ that are relevant
for the prediction of the propensity score or conditional expected outcome.

The idea is that we may view network embedding methods as black-box tools
for extracting information from the network that is relevant to solving prediction problems.
We make use of embedding based semi-supervised prediction models.
 What this means is that we assign an embedding $\embedding_{i}\in\Reals^{p}$
to each unit, and define predictors $\tilde{Q}(t_{i},\lambda_{i};\globparam^{Q})$
mapping the embedding and treatment to a prediction for $y_{i}$,
and predictor $\tilde{g}(\lambda_{i};\globparam^{g})$ mapping the
embeddings to predictions for $t_{i}$.
In this context, `semi-supervised' means that when training the model 
we do not use the labels of units in $I_{0}$, but 
we do use all other data---including the proxy structure on units in $I_{0}$.

An example clarifies the general approach.
\begin{example*} We denote the network $G_{n}$. We assume a continuous valued outcome. 
Consider the case where $\tilde{Q}(0,\cdot;\globparam^{Q}),\ \tilde{Q}(1,\cdot;\globparam^{Q})$
and $\logit\tilde{g}(\cdot;\globparam^{g})$ are all linear predictors.
We train a model with a relational empirical risk minimization procedure \cite{Veitch:Austern:Zhou:Blei:Orbanz:2018}. 
We set:
\begin{equation*}
\hat{\lambda}_{n},\hat{\gamma}_{n}^{Q},\hat{\gamma}_{n}^{g}=\argmin_{\embedding,\globparam^{Q},\globparam^{g}}\EE_{G_{k}=\samp(G_{n},k)}[L(G_{k};\lambda,\globparam^{Q},\globparam^{g})]
\end{equation*}
where $\samp(G_{n},k)$ is a randomized sampling algorithm that returns
a random subgraph of size $k$ from $G_{n}$ (e.g., a random walk with $k$ edges), and
\begin{align*}
  L(G_{k};\lambda,\globparam^{Q},\globparam^{g}) &=
                                                   \sum_{i\in I\backslash I_{0}}(y_{i}-\tilde{Q}(t_{i},\lambda_{i};\globparam^{Q}))^{2}
                                                   +\sum_{i\in I\backslash I_{0}}\crossent(t_{i},\tilde{g}(\lambda_{i};\globparam^{g})) \\
&\qquad + \sum_{i,j\in I\times I}\crossent(1[(i,j)\in G_{k}],\sigma(\embedding_{i}^{T}\embedding_{j})).
\end{align*}
Here, $I$ is the full set of units, and $1[(i,j)\in G_{k}]$
indicates whether units $i$ and $j$ are linked. Note that the final
term of the model is the one that explains the relational structure.
Intuitively, it says that the logit probability of an edge is the inner product of the embeddings of
the end points of the edge. 
This loss term makes use of the entire dataset, including
links that involve the heldout units. This is important to ensure
that the embeddings for the heldout data `match' the rest of the embeddings.
\end{example*}

\parhead{Estimation.} With a trained model in hand, computing the estimate of the treatment
effect is straightforward. Simply plug-in the estimated values
of the nuisance parameters to a standard estimator. For example,
using the A-IPTW estimator \cref{eq:psiaiptw-canon},
\begin{equation}
\label{eq:psiaiptw}
\begin{split}
\psiaiptw_{n}(I_{0}) &:= \frac{1}{\abs{I_{0}}}\sum_{i\in I_{0}}\tilde{Q}(1,\hat{\embedding}_{n,i};\hat{\globparam}_{n}^{Q})-\tilde{Q}(0,\hat{\embedding}_{n,i};\hat{\globparam}_{n}^{Q}) \\
 &\qquad +\frac{1}{\abs{I_{0}}}\sum_{i\in I_{0}} \bigg( \frac{I[t_{i}=1]}{\tilde{g}(\hat{\embedding}_{n,i};\hat{\globparam}_{n}^{g} )}- \frac{I[t_{i}=0]}{1-\tilde{g}(\hat{\embedding}_{n,i};\hat{\globparam}_{n}^{g})} \bigg)  (y_{i}-\tilde{Q}(t_{i},\hat{\embedding}_{n,i};\hat{\globparam}_{n}^{Q})).
\end{split}
\end{equation}
We also allow for a more sophisticated variant. We split the data into $K$ folds $I_{0},\dots,I_{K-1}$
and define our estimator as:
\begin{equation}
\label{k-fold-aiptw}
\psiaiptw_{n}=\frac{1}{K}\sum_{j}\psiaiptw_{n}(I_{j}).
\end{equation}
This variant is more data efficient than just using a single fold.
Finally, the same procedure applies to estimators other than
the A-IPTW. We consider the effect of the choice of estimator in \cref{sec:experiments}.

\section{Validity}
\label{sec:validity}

When does the procedure outlined in the previous section yield valid
inferences?
We now present a theorem establishing sufficient conditions. The result is an adaption of the 
``double machine learning'' of \citet{Chernozhukov:Chetverikov:Demirer:Duflo:Hansen:Newey:Robins:2017, Chernozhukov:Chetverikov:Demirer:Duflo:Hansen:Newey:2017}
to the network setting.
We first give the technical statement, and then discuss its significance and interpretation.

Fix notation as in the previous section. 
We also define $\hat{\globparam}^{Q, I^c_k}_n$
and $\hat{\globparam}^{g, I^c_k}_n$ to be the estimates for $\globparam_{Q},\globparam_{g}$
calculated using all but the $k$th data fold.
\begin{restatable}{assumption}{asetup}
\label{assumption:setup}
The  probability distributions $P$ satisfies
\begin{align*}
Y = Q(T,Z)+\zeta, \qquad &\EE[\zeta \g Z,T]=0, \\
T=g(Z)+\nu, \qquad &\EE[\nu\g Z]=0.
\end{align*}
\end{restatable}

\begin{restatable}{assumption}{aidentification}
\label{assumption:identification}
There is some function $\lambda$ mapping features $Z$ into $\Reals^p$ such that
$\lambda$ satisfies the condition of \cref{thm:prediction-suffices}, and
each of $||\tilde{Q}_n(0, \hat{\lambda}_{n,i}; \hat{\gamma}_{Q, I^c_k})-Q(0,\lambda(Z_i))||_{P,2}$, $||\tilde{Q}_n(1, \hat{\lambda}_{n,i}; \hat{\gamma}_{Q, I^c_k})-Q(1,\lambda(Z_i))||_{P,2}$, and $||\tilde{g}_n(\hat{\lambda}_{n,i};\hat{\gamma}_{g, I^c_k})-g(\lambda(Z_i))||_{P,2}$ goes to $0$ as $n\to\infty$.
Additionally, $\lambda$ must satisfy all of the following assumptions.
\end{restatable}

\begin{restatable}{assumption}{amoment}
\label{assumption:moment}
The following moment conditions hold for some fixed $\epsilon, C, c$, some $q > 4$, and all $t\in \{0,1\}$  
\begin{align*}
 ||Q(t,\lambda(Z))||_{P,q}&\leq C,\\
 ||Y||_{P,q}&\leq C,\\
 P(\varepsilon\leq g(\lambda(Z))\leq 1-\varepsilon )&=1,\\
 P(\E{P}{\zeta^2 \g \lambda(Z)}\leq C)&=1,\\
 ||\zeta||_{P,2}&\geq c, \\
 ||\nu ||_{P,2} &\geq c.
\end{align*}
\end{restatable}

\begin{restatable}{assumption}{arate}
\label{assumption:rate}
The estimators of nuisance parameters satisfy the following accuracy requirements.
There is some $\delta_n, \Delta_{n_K}\rightarrow 0$ such that
for all $n\geq 2K$ and $d\in\{0,1\}$ it holds with probability no less than $1-\Delta_{n_K}$:
\begin{equation}
\label{eq:rate}
||\tilde{Q}_n(d, \hat{\lambda}_{n,i}; \hat{\gamma}_{Q, I^c_k})-Q(d,\lambda(Z_i))||_{P,2} \cdot||\tilde{g}_n(\hat{\lambda}_{n,i} ;\hat{\gamma}_{g, I^c_k})-g(\lambda(Z_i))||_{P,2} \leq\delta_{n_K} \cdot n_K^{-1/2}
\end{equation}

And,
\begin{equation}\label{eq:est_positivity}
P(\varepsilon\leq \tilde{g}_n(\hat{\lambda}_{n,i};\hat{\gamma}_{g, I^c_k})\leq 1-\varepsilon )=1,
\end{equation}
\end{restatable}

\begin{restatable}{assumption}{adependence}
\label{assumption:dependence}
We assume the dependence between the trained embeddings is not too strong: For any $i,j$ and  all bounded continuous functions $f$ with mean 0,
\begin{align}
\E{}{f(\hat{\lambda}_{n,i})\cdot f(\hat{\lambda}_{n,j})} = o(\frac{1}{n}).
\end{align}
\end{restatable}

\begin{theorem}\label{thm:CAN}
Denote the true ATE as $\psi$.
Let $\hat{\psi}_{n}$ be the $K$-fold A-IPTW variant defined in \cref{k-fold-aiptw}.
Under \Cref{assumption:setup,assumption:identification,assumption:moment,assumption:rate,assumption:dependence}, $\hat{\psi}_n$ 
concentrates around $\psi$ with the rate
$1/\sqrt{n}$ and is approximately unbiased and normally distributed:
\begin{align*}
&\sigma^{-1} \sqrt{n} ( \hat{\psi}_{n} - \psi ) \stackrel{d}{\rightarrow} \cN ( 0 , 1 ) \\
&\sigma^2 =\E{P}{\varphi_0^2(Y, T, \lambda(Z); \theta_0, \eta(\lambda(Z)))},
\end{align*} 
where
\begin{equation*}
\begin{split}
\varphi_0(Y, T, \lambda(Z); \theta_0, \eta(\lambda(Z))) &= \frac{T}{g(\lambda(Z))}\{Y- Q(1,\lambda(Z))\} - \frac{1-T}{1-g(\lambda(Z))}\{Y-Q(0,\lambda(Z))\} \\
 &\qquad + \{Q(1,\lambda(Z))-Q(0,\lambda(Z))\} - \psi.
\end{split}
\end{equation*}
\end{theorem}

\begin{proof}
The proof follows \citet{Chernozhukov:Chetverikov:Demirer:Duflo:Hansen:Newey:2017}.
The main changes are technical modifications exploiting
\cref{assumption:dependence} to allow for the use of the full data in the embedding training. 
We defer the proof to the appendix.
\end{proof}

\parhead{Interpretation and Significance.}
\cref{thm:CAN} promises us that, under suitable conditions, the 
treatment effect is identifiable and can be estimated at a
fast rate. It is not surprising that there are some conditions 
under which this holds. The insight from \cref{thm:CAN} 
lies with the particular assumptions that are required.

\cref{assumption:setup,assumption:moment} are standard conditions.
\cref{assumption:setup} posits a causal model that 
(i) restricts the treatments and outcomes to a pure unit effect (i.e., 
it forbids contagion effects), and that
(ii) renders the causal effects identifiable when $Z$ observed.
\cref{assumption:moment} is technical conditions on the data generating distribution.
This assumption includes the standard positivity condition. 
Possible violations of these conditions are important 
and must be considered carefully in practice.
However, such considerations are standard, independent of the non-iid, no-generative-model setting 
that is our focus, so we do not comment further.

Our first deviation from the standard causal inference setup is \cref{assumption:identification}.
This is the identification condition when $Z$ is not observed.
It requires that the learned embeddings are able to extract whatever information is relevant
to the prediction of the treatment and outcome.
This assumption is the crux of the method.

A more standard assumption would directly posit the
relationship between $Z$ and the proxy network; e.g., by assuming a stochastic block model
or latent space model.
The practitioner is then required to assess whether the posited model is realistic.
In practice, all generative models of networks fail to capture the structure of real-world networks.
Instead, we ask the practitioner to judge the plausibility of the \emph{predictive} embedding model.
Such judgements are non-falsifiable, and must be based on experience with the methods
and trials on semi-synthetic data.
This is a difficult task, but the assumption is at least not violated a priori.

In practice, we do not expect the identification assumption to hold exactly.
Instead, the hope is that applying the method will adjust for whatever confounding information
is present in the network.
This is useful even if there is confounding exogenous to the network.
We study the behavior of the method in the presence of exogenous confounding in \cref{sec:experiments}. 

The condition in \cref{assumption:rate} addresses the
\emph{statistical} quality of the nuisance parameter estimation procedure.
For an estimator to be useful, it must produce accurate estimates 
with a reasonable amount of data. It is intuitive that if accurately estimating 
the nuisance parameters requires an enormous amount of data, 
then so too will estimation of $\psi$.
\cref{eq:rate} shows that this is not so. 
It suffices, in principle, to estimate the nuisance parameters crudely, 
e.g., a rate of $o(n^{1/4})$ each. 
This is important because the need to estimate the embeddings may 
rule out parametric-rate convergence of the nuisance parameters.
\cref{thm:CAN} shows this is not damning.

\cref{assumption:dependence} is the price we pay for training the embeddings with the full data. 
If the pairwise dependence between the learned embeddings is very strong
then the data splitting procedure does not guarantee that the estimate is valid.
However, the condition is weak and holds empirically.
The condition can also be removed by a two-stage procedure where the
embeddings are trained in an unsupervised manner and then used as a
direct surrogate for the confounders. 
However, such approaches have relatively poor predictive performance \cite{Yang:Cohen:Salakhudinov:2016,Veitch:Austern:Zhou:Blei:Orbanz:2018}. We compare to the two-stage approach in \cref{sec:experiments}.

\section{Experiments}
\label{sec:experiments}
The main remaining questions are:
Is the method able to adjust for confounding in practice?
If so, is the joint training of embeddings and classifier important?
And, what is the best choice of plug-in estimator for the second stage of the procedure?
Additionally, what happens in the (realistic) case that the network does not
carry all confounding information?

We investigate these questions with experiments on a semi-synthetic network dataset.\footnote{Code and pre-processed data at \href{https://github.com/vveitch/causal-network-embeddings}{github.com/vveitch/causal-network-embeddings}}
We find that in realistic situations, the network adjustment improves the estimation of the average treatment effect.
The estimate is closer to the truth than estimates from either a parametric baseline, or a two-stage embedding procedure.
Further, we find that network adjustment improves estimation quality even in the presence of confounding that is
exogenous to the network. That is, the method still helps even when full identification is not possible.
Finally, as predicted by theory, we find that the robust estimators are best when the theoretical assumptions hold.
However, the simple conditional-outcome-only estimator has better performance in the presence of significant exogenous confounding.

\subsection{Setup}
\parhead{Choice of estimator.} 
We consider 4 options for the plug-in treatment effect estimator.
\begin{enumerate}[nosep]
\item The conditional expected outcome based estimator,
\[
\hat{\psi}_{n}^{Q}=\frac{1}{n}\sum_{i}\left[\tilde{Q}_{n}(1,\hat{\embedding}_{n,i}; \hat{\globparam}_n)-\tilde{Q}_{n}(0,\hat{\embedding}_{n,i}; \hat{\globparam}_n) \right],
\]
which only makes use of the outcome model.
\item The inverse probability of treatment weighted estimator,
\[
\hat{\psi}_n^g = \frac{1}{n}\sum_i \left[ \frac{1[t_i = 1]}{\tilde{g}(\hat{\embedding}_{n,i}; \hat{\globparam}_n)} - \frac{1[t_i = 0]}{1-\tilde{g}(\hat{\embedding}_{n,i}; \hat{\globparam}_n)} \right] Y_i,
\]
which only makes use of the treatment model.
\item The augmented inverse probability treatment estimator $\psiaiptw_{n}$, defined in \cref{eq:psiaiptw}.
\item A targeted minimum loss based estimator (TMLE) \cite{vanderLaan:Rose:2011}.
\end{enumerate}
The later two estimators both make full use of the nuisance parameter estimates.
The TMLE also admits the asymptotic guarantees of \cref{thm:CAN} (though we only state the theorem for the simpler A-IPTW estimator). 
The TMLE is a variant designed for better finite sample performance.

\parhead{Pokec.}
To study the properties of the procedure, we generate semi-synthetic data using a real-world social network.
We use a subset of the Pokec social network. 
Pokec is the most popular online social network in Slovakia. 
For our purposes, the main advantages of Pokec are: 
the anonymized data are freely and openly available \cite{Takac:Zabovsky:2012,SNAP} \footnote{\url{snap.stanford.edu/data/soc-Pokec.html}},
and the data includes significant attribute information for the 
users, which is necessary for our simulations.
We pre-process the data to restrict to three districts (\v{Z}ilina, Cadca, Namestovo),
all within the same region (\v{Z}ilinsk\'{y}). 
The pre-processed network has 79 thousand users connected by 1.3 million links.

\parhead{Simulation.} 
We make use of three user level attributes in our simulations:
the \texttt{district} they live in, the user's \texttt{age}, and their Pokec \texttt{join date}.
These attributes were selected because they have low missingness and
have some dependency with the the network structure.
We discretize age and join date to a 3-level categorical variable (to match district).

For the simulation, we take each of these attributes to be the hidden confounder.
We will attempt to adjust for the confounding using the Pokec network.
We take the probability of treatment to be wholly
determined by the confounder $z$, with the three levels corresponding to $g(z) \in \{0.15, 0.5, 0.85 \}$.
The treatment and outcome for user $i$ is simulated from their confounding attribute $z_i$ as:
\begin{align}\label{eqn:sim-model}
t_i &= \bernDist(g(z_i)), \\
y_i &= t_i + \beta(g(z_i) - 0.5) + \epsilon_i \qquad \epsilon_i \dist N(0,1).
\end{align}
In each case, the true treatment effect is $1.0$.
The parameter $\beta$ controls the amount of confounding.

\parhead{Estimation.}
For each simulated dataset, we estimate the nuisance parameters using the procedure described in \cref{sec:estimation} with
$K=10$ folds.
We use a random-walk sampler with negative sampling with the default relational ERM settings \cite{Veitch:Austern:Zhou:Blei:Orbanz:2018}.  
We pre-train the embeddings using the unsupervised objective only, run until convergence.

\parhead{Baselines.}
We consider three baselines.
The first is the naive estimate that does not attempt to control for confounding;
i.e., $\frac{1}{m}\sum_{i:t_i=1}y_i - \frac{1}{n-m}\sum_{i:t_i=0}y_i$,
where $m$ is the number of treated individuals.
The second baseline is the two-stage procedure, where we first train the embeddings on the
unsupervised objective, freeze them, and then use them as features for the same predictor maps. 
The final baseline is a parametric approach to controlling for the confounding. 
We fit a mixed-membership stochastic block model \cite{Gopalan:Blei:2013} to the data, with 128 communities (chosen to match the embedding dimension).
We predict the outcome using a linear regression of the outcome on the community identities and the treatment. 
The estimated treatment effect is the coefficient of the treatment.

\subsection{Results}
\begin{wrapfigure}{r}{0.5\textwidth}
  \vspace{-25pt}
  \begin{center}
    \includegraphics[width=0.48\textwidth]{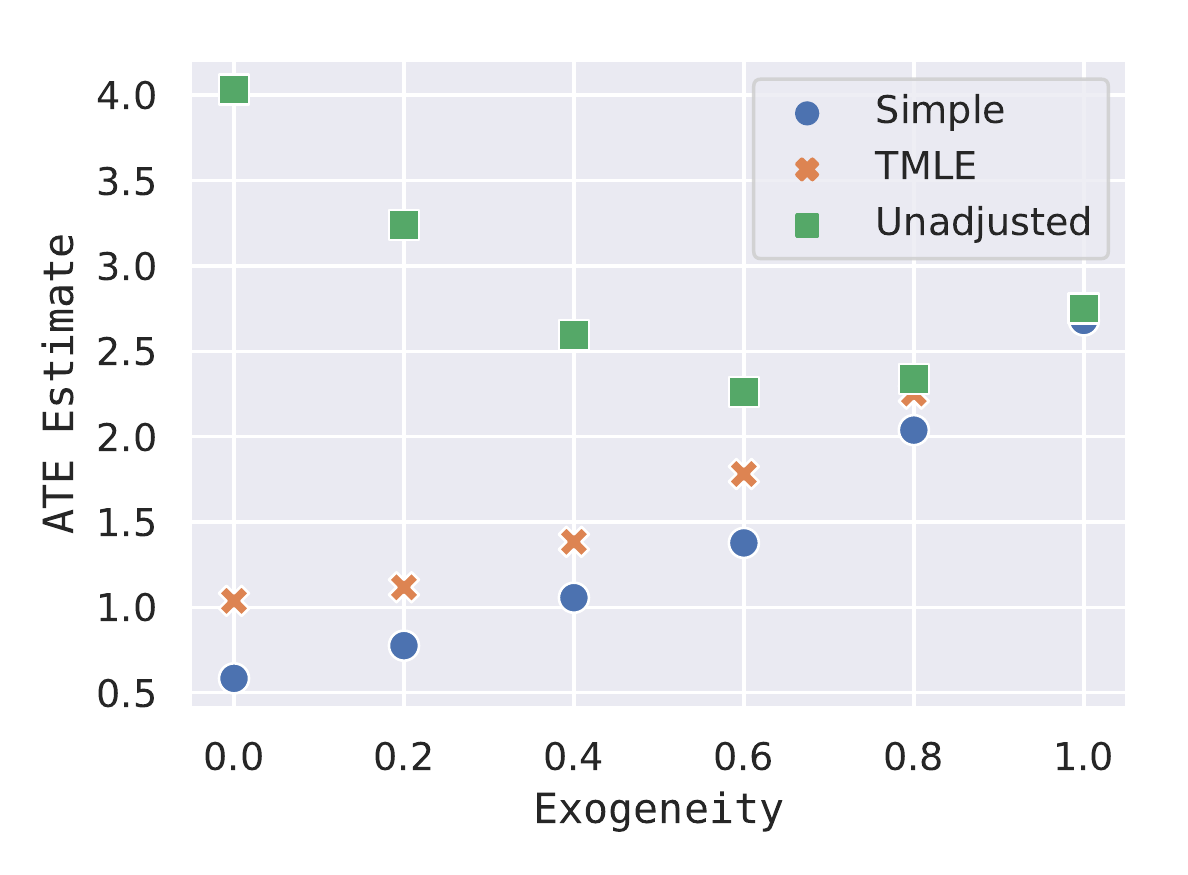}
  \end{center}
  \caption{Adjusting for the network helps even when the
    no exogenous confounding assumption is violated.
    The robust TMLE estimator is the best estimator when no assumptions are violated.
    The simple conditional-outcome-only estimator (``Simple'') is better in the presence of moderate exogeneity.  
    Plot shows estimates of ATE from district simulation.  Ground truth is 1.   \label{fig:district-exogeneity}}
   \vspace{-20pt}
\end{wrapfigure}
\parhead{Comparison to baselines.}
\sisetup{
table-number-alignment=center,
separate-uncertainty=true,
table-figures-integer = 1,
table-figures-decimal = 2}
\begin{table}
\centering
\caption{Adjusting using the network improves ATE estimate in all cases.
  Further, the single-stage method is more accurate than baselines.
  Table entries are estimated ATE with 10-fold std. Ground truth is 1.0.
Low and high confounding correspond to $\beta=1.0$ and $10.0$.}
\label{tb:baseline_comparison}
\begin{tabular}{l
                S[separate-uncertainty,table-figures-uncertainty=1]
                S[separate-uncertainty,table-figures-uncertainty=1]
                S[separate-uncertainty,table-figures-uncertainty=1]
                S[separate-uncertainty,table-figures-uncertainty=1]
                S[separate-uncertainty,table-figures-uncertainty=1]
                S[separate-uncertainty,table-figures-uncertainty=1]} 
\vspace{1pt}
  & \multicolumn{2}{c}{\texttt{age}} & \multicolumn{2}{c}{\texttt{district}} & \multicolumn{2}{c}{\texttt{join date}} \\ 
 \multicolumn{1}{r}{Conf.}  & {Low} & {High} & {Low} & {High} & {Low} & {High} \\
  \midrule
     Unadjusted & 1.32\pm0.02  & 4.34\pm0.05 & 1.34\pm0.03 & 4.51\pm0.05  & 1.29\pm0.03 & 4.03\pm0.06 \\
     Parametric & 1.30\pm0.00  & 4.06\pm0.01 & 1.21\pm0.00 & 3.22\pm0.01  & 1.26\pm0.00 & 3.73\pm0.01  \\
     Two-stage & 1.33\pm0.02  & 4.55\pm0.05 & 1.34\pm0.02 & 4.55\pm0.05  & 1.30\pm0.03 & 4.16\pm0.06  \\
     $\psiaiptw_n$ & \maxf{1.24\pm0.04} & \maxf{3.40\pm0.04} & \maxf{1.09\pm0.02} & \maxf{2.03\pm0.07} & \maxf{1.21\pm0.05} & \maxf{3.26\pm0.09} 
\end{tabular}
\end{table}
We report comparisons to the baselines in \cref{tb:baseline_comparison}.
As expected, adjusting for the network improves estimation in every case.
Further, the one-stage embedding procedure is more accurate than baselines.

\parhead{Choice of estimator.}
\begin{table}
\centering
\caption{The conditional-outcome-only estimator is usually most accurate.
  Table entries are estimated ATE with 10-fold std. Ground truth is 1.0.
Low and high confounding correspond to $\beta=1.0$ and $10.0$.}
\label{tb:estimator_comparison}
\begin{tabular}{l
                S[separate-uncertainty,table-figures-uncertainty=1]
                S[separate-uncertainty,table-figures-uncertainty=1]
                S[separate-uncertainty,table-figures-uncertainty=1]
                S[separate-uncertainty,table-figures-uncertainty=1]
                S[separate-uncertainty,table-figures-uncertainty=1]
                S[separate-uncertainty,table-figures-uncertainty=1]} 
\vspace{1pt}
  & \multicolumn{2}{c}{\texttt{age}} & \multicolumn{2}{c}{\texttt{district}} & \multicolumn{2}{c}{\texttt{join date}} \\ 
 \multicolumn{1}{r}{Conf.}  & {Low} & {High} & {Low} & {High} & {Low} & {High} \\
  \midrule
  $\hat{\psi}^Q_n$ & \maxf{1.05\pm0.24} & \maxf{2.77\pm0.35} & \maxf{1.03\pm0.25} & 1.75\pm0.20 & \maxf{1.17\pm0.35} & \maxf{2.41\pm0.45} \\ 
  $\hat{\psi}^g_n$ & 1.27\pm0.03 & 3.12\pm0.06 & 1.10\pm0.03 & \maxf{1.66\pm0.07} & 1.29\pm0.05 & 3.10\pm0.07 \\
  $\psiaiptw_n$ & 1.24\pm0.04 & 3.40\pm0.04 & 1.09\pm0.02 & 2.03\pm0.07 & 1.21\pm0.05 & 3.26\pm0.09 \\
  $\hat{\psi}^{\mathrm{TMLE}}_n$ & 1.21\pm0.03 & 3.26\pm0.07 & 1.09\pm0.04 & 2.02\pm0.05 & 1.20\pm0.05 & 3.13\pm0.09 \\

\end{tabular}
\end{table}
We report comparisons of downstream estimators in \cref{tb:estimator_comparison}.
The conditional-outcome-only estimator usually yields the best estimates,
substantially improving on either robust method.
This is likely because the network does not carry all information about the confounding factors,
violating one of our assumptions.
We expect that \texttt{district} has the strongest dependence with the network,
and we see best performance for this attribute.
Poor performance of robust estimators when assumptions are violated has been observed in other contexts \cite{Kang:Schafer:2007}.

\parhead{Confounding exogenous to the network.}
In practice, the network may not carry information about all sources of confounding.
For instance, in our simulation, the confounders may not be wholly predictable from the network structure.
We study the effect of exogenous confounding by
a second simulation where the confounder consists of a part that can be fully inferred
from the network and part that is wholly exogenous.

For the inferrable part, we use the estimated propensity scores $\{\hat{g}_i\}$ from the
\texttt{district} experiment above.
By construction, the network carries all information about each $\hat{g}_i$.
We define the (ground truth) propensity score for our new simulation as 
$\logit g_{\text{sim}} = (1-p)\logit\hat{g}_i + p \xi_i$, with $\xi_i \distiid \mathrm{N}(0,1)$.
The second term, $\xi_i$, is the exogenous part of the confounding.
The parameter $p$ controls the level of exogeneity. We simulate treatments and outcomes as in \cref{eqn:sim-model}.

In \cref{fig:district-exogeneity} we plot the estimates at various levels of exogeneity.
We observe that network adjustment helps even when the no exogenous confounding assumption is violated.
Further, we see that the robust estimator has better performance when $p=0$, i.e., when the assumptions
of \cref{thm:CAN} are satisfied. However, the conditional-outcome-only estimator is better if there is
substantial exogenous confounding.

\printbibliography

\appendix

\newpage
\section{Proof of Main Result}
We now give the proof of \cref{thm:CAN}, which establishes identifiability, consistency, and asymptotic normality.

Recall our setup: 
\begin{itemize}
\item $Y$: outcome; $T$: treatment; $Z$: confounder.

\item $Z$ is unobserved. We use some non-iid additional structure as a proxy.

\item $(Y_i, T_i, Z_i) \distiid P$. 

\item $Q(t,z) = \E{}{Y\g t,z}$; $g(Z) = P(T=1\g Z)$

\item The target parameter is the ATE,
\[
\psi_0 =\E{}{Q(1,Z)-Q(0,Z)}.
\]
\end{itemize}

\parhead{The estimator and the algorithm.} 
Recall that we learn the nuisance parameters $Q$, $g$, and the embeddings $\embedding$ using 
a semi-supervised embedding-based predictor.
We allow a slightly more general construction of the estimator than in the body of the paper.
In the body, we state the result only for the A-IPTW. Here, we allow any estimator that solves the efficient estimating equations.
This allows, for example, for targeted minimum loss based estimation.

Step 1. Form a $K$-fold partition; the splits are $I_k, k=1, \ldots, K$.
For each set $I_k$, let $I^c_k$ denote the units not in $I_k$.

Construct $K$ estimators $\check{\psi}(I^c_k), k=1,
\ldots, K$:
\begin{enumerate}
	\item Estimate the nuisance parameters $Q$, $g$, and the embedding $\embedding$:
\[
\hat{\eta}(I^c_k):= \left(\hat{\embedding}_i, \tilde{g}_n(\cdot;\hat{\globparam}^{g, I^c_k}_n), \tilde{Q}_n(\cdot, \cdot; \hat{\globparam}^{Q, I^c_k}_n)\right)
\]
\item $\check{\psi}(I^c_k)$ is a solution to the following equation:
\begin{align*}
\frac{1}{n_K}\sum_{i\in I_k} \varphi\left(Y_i, T_i, Z_i; \psi_0,  \hat{\embedding}_i, \tilde{g}_n(\cdot;\hat{\embedding}_i, \hat{\globparam}^{g, I^c_k}_n), \tilde{Q}_n(\cdot, \cdot; \hat{\globparam}^{Q, I^c_k}_n)\right) = 0,
\end{align*}
where the $\varphi(\cdot)$ function is the efficient score:
\begin{align*}
&\varphi(Y, T, Z; \psi_0, \embedding, \tilde{g}_n, \tilde{Q}_n) \\
=& \frac{T}{\tilde{g}_n(\embedding)}\{Y- \tilde{Q}_n(1,\embedding)\} - \frac{1-T}{1-\tilde{g}_n(\embedding)}\{Y-\tilde{Q}_n(0,\embedding)\} + \{\tilde{Q}_n(1,\embedding)-\tilde{Q}_n(0,\embedding)\} - \psi_0.
\end{align*}
We note that $\varphi$ does not depend on the unobserved $Z$. 
\end{enumerate}

Step 2. The final estimator for the ATE $\psi_0$ is
\begin{align*}
\tilde{\psi} = \frac{1}{K}\sum^K_{k=1}\check{\psi}(I^c_k).
\end{align*}

\parhead{The theorem and the proof.}

\asetup*
\aidentification*
\amoment*
\arate*
\adependence*

\begin{theorem}[Validity]
Denote the true ATE as 
\begin{align*}
\psi_0 = \E{P}{Q(1, Z) - Q(0, Z)}.
\end{align*}
Under \cref{assumption:setup,assumption:identification,assumption:moment,assumption:rate,assumption:dependence} the estimator
$\tilde{\psi}$ concentrates around $\psi_0$ with the rate
$1/\sqrt{n}$ and is approximately unbiased and normally distributed:
\begin{align*}
&\sigma^{-1} \sqrt{n} ( \tilde{\psi} - \psi_0 ) \stackrel{d}{\rightarrow} \cN ( 0 , 1 ) \\
&\sigma^2 =\E{P}{\varphi^2_0\left(W;\psi_0,\eta(\lambda(Z))\right)},
\end{align*} 
where 
\begin{align*}
W &= (Y,T,\lambda(Z)),\\
\eta(\lambda(Z)) &= (g(\lambda(Z)), Q(T,\lambda(Z))),
\end{align*} 
and
\begin{align*}
&\varphi_0(Y, T, \lambda(Z); \psi_0, \eta(\lambda(Z))) \\
=& \frac{T}{g(\lambda(Z))}\{Y- Q(1,\lambda(Z))\} - \frac{1-T}{1-g(\lambda(Z))}\{Y-Q(0,\lambda(Z))\} + \{Q(1,\lambda(Z))-Q(0,\lambda(Z))\} - \psi_0.
\end{align*}
\end{theorem}

\begin{proof}
  We prove the result for the special case where $\lambda$ is the identity map.
  By \cref{assumption:identification} this is without loss of generality---it's the case
  where all of the information in $Z$ is relevant for prediction.
  This is not an important mathematical point, but substantially simplifies notation.
  
The proof follow the same idea as in \citet{Chernozhukov:Chetverikov:Demirer:Duflo:Hansen:Newey:2017}
with a few modifications accounting for the non-iid proxy structure.

We start with some notation.
\begin{enumerate}
	\item $||\cdot||_{P,q}$ denotes the $L_q(P)$ norm. For example, for measurable $f:\mathcal{W}\stackrel{d}{\rightarrow} \mathbb{R}$,
	\begin{align*}
	||f(W)||_{P,q}:= (\int |f(w)^q\dif P(w)|)^{1/q}.
	\end{align*}
	\item The empirical process $\mathbb{G}_{n, I}(f(W))$ for $||f(W_i)||_{P,2} < \infty$ is
	\begin{align*}
	\mathbb{G}_{n, I}(f(W)):= \frac{1}{\sqrt{n}}\sum_{i\in I}(f(W_i)-\int f(w)\dif P(w)).
	\end{align*}
	\item The empirical expectation and probability is
	\begin{align*}
	\E{n,I}{f(W)} :=\frac{1}{n}f(W_i); \qquad \mathbb{P}_{n,I}(A) :=\frac{1}{n}\sum_{i\in I} 1(W_i \in A).
	\end{align*}
\end{enumerate} 
Let $\mathbb{P}_n$ be the empirical measure.

Step 1: (Main Step). Letting $\check{\psi}_{k} = \check{\psi}(I^c_k)$, we first write
\begin{align}
\label{eq:firstexpansion}
\sqrt{n} ( \check{\psi}_k - \psi_0 ) = \mathbb{G}_{n,I^c_k} \varphi ( W ; \psi_0 , \hat{\eta} ( I^c_k ) ) + \sqrt{n} \int \varphi ( w ; \psi_0 , \hat{\eta}( I^c_k ) ) \dif \mathbb{P}_n ( w ) ,
\end{align}
where
\begin{align*}
\hat{\eta}(I^c_k):= \left(\hat{\embedding}_i, \tilde{g}_n(\cdot;\hat{\globparam}^{g, I^c_k}_n), \tilde{Q}_n(\cdot, \cdot; \hat{\globparam}^{Q, I^c_k}_n)\right)
\end{align*}
as is defined earlier.

Steps 2 and 3 below demonstrate that for each $k = 1, \ldots, K$,
\begin{equation}
\int(\varphi(w;\psi_0 ,\hat{\eta}(I^c_k))-\varphi_0(w;\psi_0 ,\eta(z)))^2\dif \mathbb{P}_n(w)=o_{\mathbb{P}_n}(1), \label{eq:step2}
\end{equation}
and that
\begin{equation}
\sqrt{n} \int \varphi(w;\psi_0 ,\hat{\eta}(I^c_k))\dif \mathbb{P}_n(w) = o_{\mathbb{P}_n}(1).\label{eq:step3}
\end{equation}

\Cref{eq:step2} implies
\begin{align*}
\mathbb{G}_{n,I^c_k}\left(\varphi(w;\psi_0 ,\hat{\eta}(I^c_k))-\varphi_0(w;\psi_0 ,\eta(z))\right) = o_{\mathbb{P}_n}(1)
\end{align*}
due to Lemma B.1 of \citet{Chernozhukov:Chetverikov:Demirer:Duflo:Hansen:Newey:2017} and the Chebychev's
inequality.

We note that $\hat{\eta}(I^c_k) = \left(\hat{\embedding}_i,
\tilde{g}_n(\cdot;\hat{\globparam}^{g, I^c_k}_n), \tilde{Q}_n(\cdot,
\cdot; \hat{\globparam}^{Q, I^c_k}_n)\right)$, where the embedding
$\hat{\embedding}_i$'s are not independent. By contrast, $\eta(z)$ only
depends on $Z_i$ where all $Z_i$'s are independent.

We next show $\sigma^{-1}\sqrt{n_K}(\check{\psi}_k -\psi_0 )_{k=1}^K =\sigma^{-1}\mathbb{G}_{n,I^c_k} \varphi_0(W;\psi_0 ,\eta(Z) )_{k=1}^K +o_{\mathbb{P}_n}(1)$.

First, we notice
\begin{align*}
&\E{}{[\sqrt{n_K}(\check{\psi}_k -\psi_0 ) - \mathbb{G}_{n,I^c_k} \varphi_0(W;\psi_0 ,\eta(Z) )]^2\g I^c_k}\\
=&\E{}{[\mathbb{G}_{n,I^c_k} \varphi ( W ; \psi_0 , \hat{\eta} ( I^c_k ) ) - \mathbb{G}_{n,I^c_k} \varphi_0(W;\psi_0 ,\eta(Z) ) + o_{\mathbb{P}_n}(1)]^2\g I^c_k}\\
=&\E{}{(\mathbb{G}_{n,I^c_k} \varphi ( W ; \psi_0 , \hat{\eta} ( I^c_k ) ))^2\g I^c_k}  + \E{}{(\mathbb{G}_{n,I^c_k} \varphi_0(W;\psi_0 ,\eta(Z) ))^2\g I^c_k}\\
& - 2 \E{}{(\mathbb{G}_{n,I^c_k} \varphi ( W ; \psi_0 , \hat{\eta} ( I^c_k ) )\cdot (\mathbb{G}_{n,I^c_k} \varphi_0(W;\psi_0 ,\eta(Z) ))\g I^c_k} + o_{\mathbb{P}_n}(1)
\end{align*}

The first equality is due to \cref{eq:firstexpansion} and
\cref{eq:step2}. The second equality is due to
\begin{align}
\label{eq:empiricalmeanzero}
\E{}{\mathbb{G}_{n,I^c_k} \varphi ( W ; \psi_0 , \hat{\eta} ( I^c_k
) )} = \E{}{\mathbb{G}_{n,I^c_k} \varphi_0(W;\psi_0 ,\eta(Z) )} =
0.
\end{align}

If we write $\bar{\varphi}(W_i):= \varphi(W_i) - \int \varphi(w)\dif \mathbb{P}_n ( w ) $,
we have
\begin{align*}
&\E{}{[\sqrt{n_K}(\check{\psi}_k -\psi_0 ) - \mathbb{G}_{n,I^c_k} \varphi_0(W;\psi_0 ,\eta(Z) )]^2\g I^c_k}\\
=&\frac{1}{n}\E{}{\sum_{i,j=1}^{n_K}\bar{\varphi}(W_i ; \psi_0 , \hat{\eta} ( I^c_k )) \cdot \bar{\varphi}(W_j ; \psi_0 , {\eta} ( I^c_k ))\g I^c_k} \\
&+ \frac{1}{n}\E{}{\sum_{i,j=1}^{n_K}\bar{\varphi}_0(W_i ; \psi_0 , {\eta} ( Z_i )) \cdot \bar{\varphi}_0(W_j ; \psi_0 , \hat{\eta} ( Z_j ))}\\
&- 2 \E{}{(\mathbb{G}_{n,I^c_k} \varphi ( W ; \psi_0 , \hat{\eta} ( I^c_k ) )\g I^c_k}\cdot \E{}{(\mathbb{G}_{n,I^c_k} \varphi_0(W;\psi_0 ,\eta(Z) ))} +o_{\mathbb{P}_n}(1)\\
= & \frac{1}{n}\sum_{i,j=1}^{n_K} o(\frac{1}{n}) + \frac{1}{n}\sum_{i,j=1}^{n_K}\E{}{\bar{\varphi}_0(W_i ; \psi_0 , {\eta} ( Z_i ))} \cdot \E{}{\bar{\varphi}_0(W_j ; \psi_0 , \hat{\eta} ( Z_j ))}+ o_{\mathbb{P}_n}(1)\\
= & o_{\mathbb{P}_n}(1)
\end{align*}

The second equality is due to \Cref{assumption:dependence}, the
independence of $W_i$'s, and \Cref{eq:empiricalmeanzero}.

By Lemma B.1 of \citet{Chernozhukov:Chetverikov:Demirer:Duflo:Hansen:Newey:2017},
\begin{align*}
\E{}{[\sqrt{n_K}(\check{\psi}_k -\psi_0 ) - \mathbb{G}_{n,I^c_k}
\varphi_0(W;\psi_0 ,\eta(Z) )]^2\g I^c_k} =o_{\mathbb{P}_n}(1)
\end{align*}
implies 
\begin{align*}
\sqrt{n_K}(\check{\psi}_k -\psi_0 ) -
\mathbb{G}_{n,I^c_k} \varphi_0(W;\psi_0 ,\eta(Z) ) =
o_{\mathbb{P}_n}(1)
\end{align*}

Therefore, we have
\begin{align*}
\sigma^{-1}\sqrt{n_K}(\check{\psi}_k -\psi_0 )_{k=1}^K 
=\sigma^{-1}\mathbb{G}_{n,I^c_k} \varphi_0(W;\psi_0 ,\eta(Z) )_{k=1}^K +o_{\mathbb{P}_n}(1) \stackrel{d}{\rightarrow} (\cN_k)_{k=1}^K
\end{align*}
where $(\cN_k)_{k=1}^K$ is a Gaussian vector with independent
$\cN(0,1)$ coordinates. Using the independence of $Z_i$'s and the
central limit theorem, we have
\begin{align*}
&\sigma^{-1}\sqrt{n}(\tilde{\psi} -\psi_0 )\\
=&\sigma^{-1}\sqrt{n}(\frac{1}{K}\sum^K_{k=1}(\check{\psi}_k-\psi_0))\\
=&\frac{1}{K}\sigma^{-1}\sum^K_{k=1}\mathbb{G}_{n,I^c_k} \varphi_0(W;\psi_0 ,\eta(Z) ) +o_{\mathbb{P}_n}(1) \\
\stackrel{d}{\rightarrow}&\frac{1}{K}\sum^K_{k=1}\cN_k = \cN(0,1).
\end{align*}

Step 2: This step demonstrates \Cref{eq:step2}. Observe that for some
constant $C_\varepsilon $ that depends  only on $\varepsilon$  and
$\mathcal{P}$,
\begin{align*}
||\varphi(W;\psi_0,\hat{\eta}(I^c_k))-\varphi(W;\psi_0, \eta(Z))||_{\mathbb{P}_n,2}\leq C_\varepsilon (\mathcal{I}_1 +\mathcal{I}_2 +\mathcal{I}_3),
\end{align*}
where
\begin{align*}
\mathcal{I}_1 &= \max_{d\in\{0,1\}}|| \tilde{Q}_n(d,
Z; \hat{\globparam}^{Q, I^c_k}_n)-Q(d,Z)||_{\mathbb{P}_n ,2},\\
\mathcal{I}_2 &= || \frac{T ( Y - \tilde{Q}_n(1,
\embedding; \hat{\globparam}^{Q, I^c_k}_n) )}{\tilde{g}_n(\cdot;\hat{\globparam}^{g, I^c_k}_n)} - \frac{T ( Y - Q( 1 , Z ) )}{g(\embedding)} ||_{\mathbb{P}_n ,2},\\
\mathcal{I}_3 &=||\frac{(1-T) ( Y - \tilde{Q}_n(0,
\embedding; \hat{\globparam}^{Q, I^c_k}_n))}{1-\tilde{g}_n(\cdot;\hat{\globparam}^{g, I^c_k}_n)} - \frac{(1-T) ( Y - Q( 0 , Z ) )}{1-g(\embedding)} ||_{\mathbb{P}_n ,2},
\end{align*}

We bound $\mathcal{I}_1, \mathcal{I}_2, \text{ and } \mathcal{I}_3$ in turn. First, $\mathbb{P}_n(\mathcal{I}_1 > \delta_{n_K})\leq\Delta_{n_K}\rightarrow 0$ by \Cref{assumption:rate}, and so
$\mathcal{I}_1 = o_{\mathbb{P}_n}(1)$. Also, on the event that
\begin{align}
&\mathbb{P}_n ( \varepsilon\leq \tilde{g}_n ( Z ; I^c_k )\leq 1 - \varepsilon  ) = 1\label{eq:event}\\
&||\tilde{Q}_n(1,
\embedding; \hat{\globparam}^{Q, I^c_k}_n)-Q(1,Z)||_{\mathbb{P}_n,2} +||\tilde{g}_n(\cdot;\hat{\globparam}^{g, I^c_k}_n)-g(Z)||_{\mathbb{P}_n,2}\leq\delta_{n_K}\label{eq:evt2},
\end{align}
which happens with $P_{\mathbb{P}_n}$-probability at least $1-\Delta_{n_K}$ by \Cref{assumption:rate}, 
\begin{align*}
\mathcal{I}_2
&\leq\varepsilon^{-2}||Tg(Z)(Y -\tilde{Q}_n(1,\embedding; \hat{\globparam}^{Q, I^c_k}_n))-T\tilde{g}_n ( Z ; I^c_k )(Y -Q(1,Z))||_{\mathbb{P}_n,2}\\
&\leq \varepsilon^{-2}|| g(Z)(Q(1, Z) + \zeta - \tilde{Q}_n(1,\embedding; \hat{\globparam}^{Q, I^c_k}_n)) - \tilde{g}_n(Z; I^c_k)\zeta||_{\mathbb{P}_n,2}\\
&\leq \varepsilon^{-2}||g(Z)(\tilde{Q}_n(1,\embedding; \hat{\globparam}^{Q, I^c_k}_n)-Q(1,Z))||_{\mathbb{P}_n,2} +||(\tilde{g}_n ( Z ; I^c_k )-g(Z))\zeta||_{\mathbb{P}_n,2}\\
&\leq \varepsilon^{-2}||\tilde{Q}_n(1,\embedding; \hat{\globparam}^{Q, I^c_k}_n)-Q(1,Z)||_{\mathbb{P}_n,2} +\sqrt{C}||\tilde{g}_n ( Z ; I^c_k )-g(Z)||_{\mathbb{P}_n,2}\\
&\leq \varepsilon^{-2}(\delta_{n_K} +\sqrt{C}\delta_{n_K})\rightarrow 0,
\end{align*}

where the first inequality follows from \Cref{eq:event} and
\Cref{assumption:rate}, the second from the facts that $T \in \{0,
1\}$ and for $T = 1, Y = Q(1, Z) + \zeta$, the third from the triangle
inequality, the fourth from the facts that $\mathbb{P}_n(g(Z)\leq 1) =
1$ and $\mathbb{P}_n(\E{\mathbb{P}_n}{\zeta^2\g Z}\leq C) = 1$ in
\Cref{assumption:moment}, the fifth from \Cref{eq:evt2}, and the last
assertion follows since $\delta_{n_K} \rightarrow 0$. Hence,
$\mathcal{I}_2 = o_{\mathbb{P}_n}(1)$. In addition, the same argument
shows that $\mathcal{I}_3 = o_{\mathbb{P}_n}(1)$, and so \Cref{eq:step2} follows.

Step 3: This step demonstrates \Cref{eq:step3}. Observe that since
$\psi_0 = \E{\mathbb{P}_n} {Q(1, Z) - Q(0,Z)}$, the left-hand side of
\Cref{eq:step3} is equal to
\begin{align*}
\mathcal{I}_4 = &\sqrt{n}\int\frac{\tilde{g}_n ( Z ; I^c_k )-g(z)}{\tilde{g}_n ( Z ; I^c_k )}  \cdot (\tilde{Q}_n(1,\embedding; \hat{\globparam}^{Q, I^c_k}_n) - Q(1,z))\\
& + \frac{\tilde{g}_n ( Z ; I^c_k )-g(z)}{1-\tilde{g}_n ( Z ; I^c_k )} \cdot (Q(0, z; I^c_k) - Q(0, z))
\dif \mathbb{P}_n(z).
\end{align*}
But on the event that
\begin{align*}
\mathbb{P}_n ( \varepsilon\leq \tilde{g}_n ( Z ; I^c_k ) \leq 1 - \varepsilon  ) = 1
\end{align*}
and
\begin{align*}
\max_{d\in\{0,1\}} ||\tilde{Q}_n(d,\embedding; \hat{\globparam}^{Q, I^c_k}_n) - Q(d, Z)||_{\mathbb{P}_n,2} \cdot ||\tilde{g}_n(Z; I^c_k) - g(Z)||_{\mathbb{P}_n,2}\leq \delta_{n_K}\cdot n_K^{-1/2},
\end{align*}

which happens with $P_{\mathbb{P}_n}$-probability at least $1-\Delta_{n_K}$ by \Cref{assumption:rate},the Cauchy-Schwarz inequality implies that
\begin{align*}
\mathcal{I}_4\leq \frac{2\sqrt{n}}{\varepsilon } \max_{d\in\{0,1\}} ||\tilde{Q}_n(d,\embedding; \hat{\globparam}^{Q, I^c_k}_n)-Q(d,Z)||_{\mathbb{P}_n,2} \cdot||\tilde{g}_n ( Z ; I^c_k )-g(Z)||_{\mathbb{P}_n,2}\leq \frac{2\delta_{n_K}}{\varepsilon} \rightarrow 0,  
\end{align*}
which gives \Cref{eq:step3}.

\end{proof}

\end{document}